\documentclass[graybox]{svmult}

\pdfoutput=1

\usepackage{mathptmx}       % selects Times Roman as basic font
\usepackage{helvet}         % selects Helvetica as sans-serif font
\usepackage{courier}        % selects Courier as typewriter font
\usepackage{type1cm}        % activate if the above 3 fonts are
\usepackage{makeidx}         % allows index generation
\usepackage{graphicx}        % standard LaTeX graphics tool
\usepackage{multicol}        % used for the two-column index
\usepackage[bottom]{footmisc}% places footnotes at page bottom

\usepackage{amsmath}
\usepackage{amssymb}
\usepackage{algorithm}
\usepackage{algorithmic}
\usepackage{xspace}
\usepackage{paralist}
\usepackage{url}

\smartqed

\newcommand{\ith}[2]{\ensuremath{{#1}^{\mbox{\scriptsize #2}}}}
\newcommand{\term}[1]{\emph{#1}}
\newcommand{\eg}[0]{\emph{e.g.},\xspace}
\newcommand{\ie}[0]{\emph{i.e.},\xspace}
\newcommand{\etal}[0]{\emph{et al.}\xspace}
\newcommand{\cf}[0]{\emph{cf.}\xspace}
\newcommand{\VC}[1]{\ensuremath{\mathrm{VC}(#1)}}
\newcommand{\proj}[2]{\ensuremath{p_{#1}(#2)}}
\newcommand{\indic}[1]{\ensuremath{\mathbf{1}\left[#1\right]}}

\newcommand{\myparagraph}[1]{\textbf{#1}.}
\makeindex             % used for the subject index
\begin{document}

\title*{Bounding Embeddings of VC Classes into Maximum Classes}
% Use \titlerunning{Short Title} for an abbreviated version of
% your contribution title if the original one is too long
\author{J.~Hyam Rubinstein, Benjamin~I.~P. Rubinstein, and Peter~L. Bartlett}
\authorrunning{Rubinstein, Rubinstein, Bartlett}
\institute{J.~Hyam Rubinstein \at Department of Mathematics \& Statistics, The University of Melbourne, Australia\\ \email{rubin@ms.unimelb.edu.au}
\and Benjamin~I.~P. Rubinstein \at Department of Computing \& Information Systems, The University of Melbourne, Australia\\
\email{ben@bipr.net}
\and Peter~L. Bartlett \at Depts. Electrical Engineering \& Computer Sciences and Statistics, UC Berkeley, USA\\
Faculty of Science and Engineering, Queensland University of Technology, Australia\\
\email{bartlett@cs.berkeley.edu}}

\maketitle

% first for online publication
\abstract*{
One of the earliest conjectures in computational learning theory---the Sample Compression conjecture---asserts that concept classes (equivalently set systems) admit compression schemes of size linear in their VC dimension. To-date this statement is known to be true for maximum classes---those that possess maximum cardinality for their VC dimension. The most promising approach to positively resolving the conjecture is by embedding general VC classes into maximum classes without super-linear increase to their VC dimensions, as such embeddings would extend the known compression schemes to all VC classes. We show that maximum classes can be characterised by a local-connectivity property of the graph obtained by viewing the class as a cubical complex. This geometric characterisation of maximum VC classes is applied to prove a negative embedding result which demonstrates VC-$d$ classes that cannot be embedded in any maximum class of VC dimension lower than $2d$. On the other hand, we show that every VC-$d$ class $C$ embeds in a VC-$(d+D)$ maximum class where $D$ is the deficiency of $C$, \ie the difference between the cardinalities of a maximum VC-$d$ class and of $C$.  For VC-$2$ classes in binary $n$-cubes for $4 \le n \le 6$, we give best possible results on embedding into maximum classes.  For some special classes of Boolean functions, relationships with maximum classes are investigated. Finally we give a general recursive procedure for embedding VC-$d$ classes into VC-$(d+k)$ maximum classes for smallest $k$.

}

% second to actually show up in chapter
\abstract{
One of the earliest conjectures in computational learning theory---the Sample Compression conjecture---asserts that concept classes (equivalently set systems) admit compression schemes of size linear in their VC dimension. To-date this statement is known to be true for maximum classes---those that possess maximum cardinality for their VC dimension. The most promising approach to positively resolving the conjecture is by embedding general VC classes into maximum classes without super-linear increase to their VC dimensions, as such embeddings would extend the known compression schemes to all VC classes. We show that maximum classes can be characterised by a local-connectivity property of the graph obtained by viewing the class as a cubical complex. This geometric characterisation of maximum VC classes is applied to prove a negative embedding result which demonstrates VC-$d$ classes that cannot be embedded in any maximum class of VC dimension lower than $2d$. On the other hand, we show that every VC-$d$ class $C$ embeds in a VC-$(d+D)$ maximum class where $D$ is the deficiency of $C$, \ie the difference between the cardinalities of a maximum VC-$d$ class and of $C$.  For VC-$2$ classes in binary $n$-cubes for $4 \le n \le 6$, we give best possible results on embedding into maximum classes.  For some special classes of Boolean functions, relationships with maximum classes are investigated. Finally we give a general recursive procedure for embedding VC-$d$ classes into VC-$(d+k)$ maximum classes for smallest $k$.

}

\section{Introduction}

Sauer's Lemma, discovered first by Vapnik \& Chervonenkis~\cite{VC71}
and later independently by Shelah~\cite{SS72} and Sauer~\cite{S72},
upper-bounds the cardinality of a set system in terms of its Vapnik-Chervonenkis (VC) dimension.
The lemma has found many applications in such diverse fields as computational learning theory and empirical process theory~\cite{BEHW89,VC71,BartlettBook,Haussler90probablyapproximately,AngluinSurvey,DevroyeBook,vanDerVaartBook},
coding theory~\cite{Guruswami2010}, computational geometry~\cite{HausslerWelzl86,Matousek94,SetCover95,Kleinberg97}, road network routing~\cite{AbrahamDFGW11}, and automatic
verification~\cite{SetCover95}; in the former it is the avenue through which
the VC dimension enters into generalisation error bounds and the theoretical foundations
of learnability.

Maximum classes are concept classes on
the $n$-cube\footnote{As discussed below, we consider concept classes evaluated on finite samples. Such projections are equivalent to subsets of the $n$-cube. Thus we discuss concept classes as such subsets without loss of generality.} that meet Sauer's Lemma with equality~\cite{W87,KW07}: they
maximise cardinality over all concept classes with a given VC dimension.
Recent work has illuminated a beautiful geometric structure to maximum
classes, one in which such classes (and their complements) can be
viewed as complete collections of cubes---unions of ${n \choose d}$
cubes each varying over a unique set of $d$ coordinates---which
forms a $d$-contractible simplicial complex (the higher-order cubical
generalisation of a tree)~\cite{RBR08}. Another important family of
concept classes are known as maximal classes, which cannot be expanded
without increasing their VC dimension~\cite{W87,KW07}; the complement of
any maximal VC-$d$ class is also a complete collection of $(n-d-1)$-cubes~\cite{RBR08}. Indeed it is most natural to study the complementary structure
of VC classes due to these cubical characterisations.

Our key motivation for studying maximal and maximum classes
is for resolving the Sample Compression conjecture~\cite{LW86,W03}, 
a problem that has evaded resolution for over a quarter century, and that 
equates learnability with a concept class admitting a so-called finite
compression scheme. Littlestone \& Warmuth~\cite{LW86}, after showing that
finite compression schemes immediately lead to risk bounds, posed the conjecture
to determine whether the converse holds: does finite VC dimension imply
$O(d)$-sized compression schemes.
Beyond providing a deeper understanding of the fundamental notions of learning theory, such
as VC dimension, maximum and maximal classes, foundational work on the
Sample Compression conjecture may lead to practical learning algorithms.
Previously, compression-based learning algorithms~\cite{LBS04} and bounds~\cite{Lang05}
have enjoyed successful application in practice.

To date, most progress towards the conjecture has been on compressing maximum
classes. Floyd~\cite{F89} first compressed maximum classes with labeled schemes.
Later Ben-David \& Litman~\cite{BDL98} proved existence of unlabeled schemes for maximum
classes, followed by Kuzmin \& Warmuth~\cite{KW07} and Rubinstein \& Rubinstein~\cite{Rubinsteins09} who constructed unlabeled schemes using the
cubical structure of such classes. In the related problem of teaching,
Doliwa \etal~\cite{Doliwa2010} showed that the recursive teaching dimension of maximum classes
coincides with the VC dimension, using the cubical corner-peeling compression
scheme of Rubinstein \& Rubinstein~\cite{Rubinsteins09}.
Recently Livni \& Simon~\cite{honest} developed a new approach using ideas from
model theory to form bounded-size compression schemes for a new family of concept classes.
It is unclear, however, how to directly extend
any of these results to $O(d)$ schemes for general VC-classes.

To compress general classes it is necessary and sufficient
to compress maximal classes, since any concept class can be expanded to a maximal
class without increasing VC dimension. Given the past success at compressing
maximum classes, a natural approach to the conjecture is
to develop techniques for embedding any maximal class into a maximum class
without significantly increasing its VC dimension~\cite{F89}. This chapter provides
results relating to this approach.

We first discuss a series of higher-dimensional analogs of Sauer's
Lemma for counting $k$-dimensional hypercubes in the complements of general
VC-$d$ classes for $0\leq k< n-d-1$. Where Sauer's Lemma lower bounds
points (the $k=0$ case) in the complement, these higher-dimensional
analogues lower bound edges ($k=1$) all the way up to faces ($k=n-d-2$).
Moreover we show that maximum classes uniquely meet
each higher-dimensional bound with equality, just as in the $k=0$
case. These bounds were first obtained by Kuzmin \& Warmuth~\cite{KW07}. We present a different treatment 
as we are particularly interested in the graph obtained by considering only the incidence
relations of maximal-dimensional cubes along their faces of co-dimension one. 

We view this characterisation of maximum VC classes as providing
a measure of closeness of any VC class---most importantly maximal classes---to
being maximum. Knowing
how close a maximal class is to being maximum may prove to be useful
in achieving the desired maximum-embedding of maximal classes with linear increase to VC dimension.

The deficiency $D$ of a VC-$d$ class $C$ is defined as the difference between the cardinality of a maximum VC-$d$ class and of $C$---clearly maximum classes are precisely those of deficiency $0$. We prove that classes of small deficiency have useful compression schemes coming from embedding into maximum classes, by establishing that every VC-$d$ class with deficiency $D$ embeds in a maximum class of VC dimension $d+D$. 
There are two interesting steps to show this. The first is that if a VC-$d$ class $C$ projects onto a maximum class, via a projection of the binary $n$-cube to a binary $(n-k)$-cube, then $C$ embeds in a maximum VC-$(d+k)$ class. Secondly, if $C$ is a VC-$d$ class which is not maximum, there is always a projection from the binary $n$-cube to the binary $(n-1)$-cube, which reduces the deficiency of $C$. 

As an application of the characterisation of maximum VC classes, we produce a collection of concept classes of VC-dimension $d$ embedded in 
an $n$-cube, so that each class cannot be embedded in any maximum class of VC-dimension $2d-1$ but can be embedded 
in a maximum class of VC-dimension $2d$. The cubical structure of the complements is the key to the construction. This negative result
improves that of Rubinstein \& Rubinstein~\cite{Rubinsteins09}, where it is shown that for all constants $c$ there exist VC-$d$ classes which cannot be embedded in any maximum class of VC-dimension $d+c$. Our new negative result proves that while the general
Sample Compression conjecture---that every VC-$d$ class has a compression scheme of size $O(d)$---may still
hold, the constant must at least be $2$, if the compression scheme is to be obtained via embeddings.

We also give a recursive scheme to embed any VC-$d$ class into a maximum VC-$(d+k)$ class, if any such embedding exists. The scheme does not resolve the conjecture, because $k$ must be supplied, but rather demonstrates a possible approach to the compression problem, via embedding into maximum classes. The key idea is to use lifting~\cite{RBR08}.

For the special case of VC-$2$ classes in the binary $n$-cube, for $4 \le n \le 6$ we give best-possible results for embedding into maximum classes. Maximal VC-$2$ classes in the binary $4$-cube are classified. For symmetric Boolean functions, we show that there is a natural way of enlarging the class to a maximum class of the same VC dimension. A construction is given for sets of Boolean functions on $n$ variables which give maximum classes in the binary $2^n$-cube. 

\myparagraph{Chapter Organisation} We begin with preliminaries in Sect.~\ref{sec:background}. Our proof bounding the number of hypercubes contained in the complement of a VC class is presented in Sect.~\ref{sec:counting}. We then develop a new characterisation of maximum classes in Sect.~\ref{sec:iterated}. In Sect.~\ref{sec:deficiency}, we prove that every VC-$d$ class embeds in a maximum class of VC dimension $d+D$ where $D$ is the deficiency of the class. Section~\ref{sec:examples} presents examples which demonstrate a new negative result on embedding maximal classes into maximum ones in which their VC dimension must double. Section~\ref{sec:max} gives a general recursive construction of embeddings of VC-$d$ classes into VC-$(d+k)$ maximum classes. In Sect.~\ref{sec:VC-2}, classes of VC dimension $2$ embedded in binary $n$-cubes for $4 \le n \le 6$ are discussed. In Sect.~\ref{sec:functions}, symmetric and Boolean functions are viewed as classes in the binary $2^n$-cube and related to maximum classes. 
Sect.~\ref{sec:conc} concludes the chapter.

\section{Background and Definitions}\label{sec:background}

Consider the binary $n$-cube $\{0,1\}^n$ for integer $n>1$. We call any subset $C\subseteq\{0,1\}^n$ 
a \term{concept class} and elements $c\in C$ \term{concepts}. This terminology derives from statistical learning theory: a binary classifier $f:\mathcal{X} \to \{0,1\}$ on some domain $\mathcal{X}$ (\eg Euclidean space) is equivalent to the $n$-bit vector of its evaluations on a particular sample of points $X_1,\ldots,X_n \in\mathcal{X}$ of interest. Hence on a given sample we equate concepts with such classifiers, and families of classifiers (\eg the linear classifiers) with concept classes. Equivalently a concept class corresponds to a set system with underlying set taken to be the axes (or $n$ points) and each subset corresponding to the support of a concept.

\subsection{Special Concept Classes}

We next outline a number of families of concept classes central to VC theory,
and that exhibit special combinatorial structure. We begin with the important combinatorial parameter known as the VC dimension~\cite{VC71}.

\begin{definition}
The \term{Vapnik-Chervonenkis} (VC) dimension of concept class $C\subseteq\{0,1\}^n$ is defined as
\begin{eqnarray*}
\VC{C} &=& \max\left\{|I| : I\subseteq[n], \proj{I}{C}=\{0,1\}^{|I|}\right\}\enspace,
\end{eqnarray*}
where $\proj{I}{C}=\left\{(c_i)_{i\in I} \mid c\in C\right\}$ is the set of coordinate projections of the concepts of $C$
on coordinates $I\subseteq [n]$.
\end{definition}

In words, the VC dimension is the largest number $d$ of coordinates on which the restriction of the concept class forms the complete binary $d$-cube. The VC dimension is used extensively in statistical learning theory and empirical process theory to measure the complexity of families of classifiers in order to derive risk bounds. It enters into such
results via the following bound on concept class cardinality first due to Vapnik \& Chervonenkis~\cite{VC71}, and later independently by Shelah~\cite{SS72} and Sauer~\cite{S72}.

\begin{lemma}[Sauer's Lemma]
The cardinality of any concept class $C\subseteq\{0,1\}^n$ is bounded by 
\begin{eqnarray*}
|C| &\leq& \sum_{i=0}^{\VC{C}} {n \choose i}\enspace.
\end{eqnarray*}
\end{lemma}

Any concept class that meets Sauer's Lemma with equality is called \term{maximum}, while any concept class that cannot be extended without increasing VC dimension is called \term{maximal}~\cite{W87,KW07}. Trivially maximum classes are maximal by definition, while not all maximal classes are maximum~\cite{Welzl87,KW07}.

A family of ``canonical'' maximum classes, which are particularly convenient to work with, are the fixed points of a certain type of contraction-like mapping known as shifting which is used to prove Sauer's Lemma~\cite{HLW94,RBR08}.\footnote{We use $\indic{p}$ to denote the indicator function on predicate $p$, and $[n]$ to denote integers $\{1,\ldots,n\}$.}

\begin{definition}
A concept class $C\subseteq\{0,1\}^n$ is called \term{closed-below} if $c\in C$ implies that for every $I\subseteq [n]$ the concept $c_I$, with
$c_{I,i} = \indic{i\in I} c_i$, is also in $C$.
\end{definition}

We can now define the deficiency of any VC-$d$ class. 

\begin{definition}
The \term{deficiency} of a concept class $C\subseteq\{0,1\}^n$ is the difference $D=|C^\star| -|C|$ where $C^\star$ is any maximum class with the same VC dimension as $C$.
\end{definition}

\subsection{Cubical View of VC Classes}
\label{sec:cubical-bg}

Rubinstein \etal~\cite{RBR08} established the following natural geometric characterisations of VC classes, and maximum \& maximal classes in particular. 

\begin{definition}
A collection of subcubes $\mathcal{C}$ of cardinality $n \choose d$ is called \term{$d$-complete} if for all sets $I\subseteq[n]$ of cardinality $d$, there exists a $d$-cube $S_I\in\mathcal{C}$ such that $\proj{I}{S_I}=\{0,1\}^d$.
\end{definition}

\begin{theorem}\label{thm:complement}
$C\subseteq\{0,1\}^n$ has $\VC{C}\leq d$ iff $\overline{C}$ contains a $(n-d-1)$-complete collection of subcubes. In particular $\VC{C}=d$ iff $\overline{C}$ contains a $(n-d-1)$-complete collection but no $(n-d)$-complete collection. It follows that $C\subseteq\{0,1\}^n$ of VC-dimension $d$ is maximal iff $\overline{C}$ is a $(n-d-1)$-complete collection and properly contains no $(n-d-1)$-complete collection; and $C\subseteq\{0,1\}^n$ of VC-dimension $d$ is maximum iff $\overline{C}$ is the union of a maximally overlapping $(n-d-1)$-complete collection, or equivalently iff $C$ is the union of a maximally overlapping $d$-complete collection.
\end{theorem}

Due to this characterisation, it is often more convenient to focus on the complementary class $\overline{C}=\{0,1\}^n\backslash C$ of a concept class $C\subseteq\{0,1\}^n$.

Given a class $C\subseteq\{0,1\}^n$ and a projection $p_I$ from the $n$-cube to an $(n-1)$-cube on $[n]\setminus\{x\}$: the \term{tail} of $C$ with respect to $x$ is the subset of $C$ with unique images under $p_I$; the \term{reduction} $C^x$ of $C$ is the projection of the subset of $C$ with non-unique images. Welzl~\cite{W87} (\cf also Kuzmin \& Warmuth~\cite{KW07}) showed that $p_I(C)$ is a maximum class of VC-dimension $d$ while $C^x$ is a maximum class of VC-dimension $d-1$. Moreover $C^x$ is a collection of $(d-1)$-cubes which are faces of the $d$-cubes which make up $p_I(C)$.

We next review a technique due to Rubinstein \& Rubinstein~\cite{Rubinsteins09} for building 
all VC-$d$ maximum classes by starting with a 
closed-below $d$-maximum class and proceeding through
a sequence of $d$-maximum classes (inverting the process of shifting).
Lifting is the process of reconstructing $C$ from the knowledge of the tail $C_T$ and reduction $C_R$. First, we form a new maximum class $C^\prime$ in the $n$-cube by placing all the $d$-cubes with at least one vertex in $C_T$ at the level where the \ith{i}{th} coordinate $x_i=0$ and $C_R$ is used to form $d$-cubes of the form $c \times \{0,1\}$ where $c$ is a $(d-1)$-cube of $C_R$ and $\{0,1\}$ give both choices $x_i =0,1$. Now splitting $C^\prime$ along $C_R \times \{0,1\}$, each connected component of $d$-cubes, each with at least one vertex in $C_T$, is lifted to either the level $x_i=0$ or to the level $x_i=1$. Lifting all the components in this way always produces a maximum class $C$ and all maximum classes are obtained in this way by a series of lifts starting at the closed-below maximum class.

\subsection{The Sample Compression Conjecture}

Littlestone and Warmuth's
\term{Sample Compression conjecture} predicts that any concept class with VC-dimension $d$ admits a so-called compression scheme of size $O(d)$~\cite{LW86,W03}. 

\begin{definition}
Let $k\in\mathbb{N}$, domain $\mathcal{X}$, and family
of classifiers $\mathcal{F}\subseteq\{0,1\}^\mathcal{X}$.
The following pair of mappings $(\kappa_\mathcal{F}, \rho_\mathcal{F})$ is called a
\term{compression scheme} for $\mathcal{F}$ of size $k$
\begin{eqnarray*}
\kappa_\mathcal{F}: && \bigcup_{n=k}^\infty \left(\mathcal{X}\times \{0,1\}\right)^n \to \bigcup_{l=0}^k \left(\mathcal{X}\times \{0,1\}\right)^l \\
\rho_\mathcal{F}: && \left(\bigcup_{l=0}^k \left(\mathcal{X}\times \{0,1\}\right)^l\right)\times \mathcal{X} \to \{0,1\} \enspace,
\end{eqnarray*}
if they satisfy the following condition for each
classifier $f\in\mathcal{F}$ and unlabeled sample $x\in\bigcup_{n=k}^\infty \mathcal{X}^n$: we first
evaluate the \term{compression function} $\kappa_\mathcal{F}$ on $x$ labeled by $f$ to a subsequence $r$ of length at most $k$, called the \term{representative} of $f$; and then the \term{reconstruction function} $\rho_\mathcal{F}(r,\cdot)$ can label $x_i$ consistently with $f(x_i)$ for each $i\in[n]$.
\end{definition}

Floyd~\cite{F89} in 1989 showed that all VC-$d$ maximum classes can be compressed with schemes of size $d$. Since then, little progress has
been made on compressing richer families of VC classes, although unlabeled compression schemes, relations to teaching, and a number of beautiful related combinatorial results have been developed~\cite{BDL98,KW07,RBR08,LBS04,RR08,Rubinsteins09,Doliwa2010,honest}. Since concept classes inherit the compression schemes of larger
classes in which they can be embedded, a leading approach to positively establishing the conjecture is to embed (general) VC classes into maximum classes without significantly increasing VC dimension. In particular, it would be sufficient to embed any $d$-maximal class into an $O(d)$-maximum class.

\section{Bounding the Number of Hypercubes of a VC Class}\label{sec:counting}

As discussed, a natural approach to understanding the content
of a class provided by its VC dimension is via the class's cubical 
structure. In this section we focus on counting the cubes of a VC-class.

The following was established by Kuzmin \& Warmuth~\cite{KW07} via a different argument. 
We will apply this result to proving a new characterisation of maximum
classes in the next section (Theorem~\ref{thm:reduction}).

\begin{theorem}\label{thm:main}
Let integers $n,d,k$ be such that $n>1$, $0\leq d\leq n$ and $0\leq k < n-d-1$.
For any maximal concept class $C\subseteq\{0,1\}^{n}$ of VC-dimension $d$,
the number of $k$-cubes contained in $\overline{C}$ is lower bounded
by $\sum_{i=k}^{n-d-1}{i \choose k}{n \choose i}$, and the bound
is met with equality iff $C$ is maximum.
\end{theorem}

To prove this result, we first count the number of cubes in 
maximum closed-below classes. 

\begin{lemma}
Let $C$ be a maximum closed-below class of VC-dimension $d$ in the
$n$-cube. Then $C$ contains $\sum_{i=k}^{d}{i \choose k}{n \choose i}$
$k$-cubes for each $0\leq k\leq d$.
\end{lemma}

\begin{proof}
For each $d$, the maximum closed-below class of VC-dimension $d$ is 
the class with all concepts with $\ell_1$-norms at most $d$~\cite{RBR08}.
(In other words, all the concepts are binary strings of length $n$ in the $n$-cube
with at most $d$ ones).

For $k=0$ we must count the number of points in $C$. This is done
by simply partitioning the vertices of $C$ into layers, where each
layer contains vertices with the same $\ell_{1}$-norm. (In other words, 
the same number of ones). At the top $d$
layer there are ${n \choose d}$ nodes of norm $d$, at layer $d-1$
there are ${n \choose d-1}$ nodes, etc. down to the bottom $0$ layer
which consists of a single vertex of zero norm.

The $k=1$ case corresponds to the edge counting argument in bounding
the density of one-inclusion graphs~\cite{HLW94,RBR08}, which is one of the steps 
used in proving Sauer's Lemma by shifting. By noting that every
edge connects one vertex with lower norm to a vertex with higher norm,
we may count edges uniquely by considering edges oriented downwards,
and again partitioning them by the norm of the higher incident vertex.
At the top $d$ layer each of the ${n \choose d}$ vertices identifies
$d={d \choose 1}$ edges, at the next $d-1$ layer each of the ${n \choose d-1}$
vertices identifies $d-1={d-1 \choose 1}$ edges, etc. all the way
down to the first layer where each of the ${n \choose 1}$ vertices
identifies $1={1 \choose 1}$ edge.

For the general $k>1$ case the argument remains much the same. Now
instead of orienting edges away from their top incident vertex, we
orient $k$-cubes away from their top incident vertex; where each
edge is identified by specifying the top and bottom vertices, each
$k$-cube is identified by specifying the top vertex and each of its
$k$ neighboring vertices in the $k$-cube. We again partition the
$k$-cubes by the layers of their top vertices. The top $d$ layer
contains ${n \choose d}$ vertices each of which identifies ${d \choose k}$
$k$-cubes, the $d-1$ layer contains ${n \choose d-1}$ vertices
each identifying ${d-1 \choose k}$ $k$-cubes, all the way down to
the $k$ layer which contains ${n \choose k}$ vertices each identifying
$1={k \choose k}$ $k$-cubes.
\qed\end{proof}

We may now prove the main result of this section.

\begin{theopargself}
\begin{proof}[of Theorem~\ref{thm:main}]
Consider the technique of lifting (as reviewed in Sect.~\ref{sec:cubical-bg}): it is obvious that the lifting process does not change the number of $k$-cubes for all $k$ with $0 \leq k <d$. 
And since lifting always creates maximum classes, and all
such classes are created by lifting, it follows that all maximum classes of VC-dimension $d$ have the same number of $k$-cubes as the closed-below maximum classes of VC-dimension $d$.

The final step is to show that for any class $C$ in the $n$-cube which is not maximum,  $\overline C$ must have more $k$-cubes than a maximum class of VC-dimension $n-d-1$, for all $k$ satisfying $0 \leq k < n-d-1$. This can be established using shifting---the inverse
process to lifting where all points move along a chosen dimension towards zero provided no existing points block movement~\cite{H95}. Namely we know that $\overline C$
is a complete union of $(n-d-1)$-cubes, since $C$ is maximal with VC-dimension $d$. It is convenient to shift $(n-d-1)$-cubes rather than vertices. Namely for the \ith{i}{th} coordinate, we can shift an $(n-d-1)$-cube of $\overline C$ with anchor containing this coordinate and having value $x_i=1$ to the value $x_i=0$. Notice that this type of shifting preserves the number of $(n-d-1)$-cubes but may decrease the number of lower-dimensional cubes. In fact, since by assumption $C$ is not maximum, neither is $\overline C$. So during the shifting process, the number of vertices must decrease, \ie two vertices which differ only at the \ith{i}{th} coordinate become identified. But then it is easy to see that the number of $k$-cubes decreases for all $k$ with $0 \leq k < n-d-1$ by considering $k$-cubes having one or other of these two vertices. This completes the proof.
\qed\end{proof}
\end{theopargself}

\section{An Iterated-Reduction Characterisation of Maximum Classes}\label{sec:iterated}

In this section, we offer another characterisation of maximum classes (\cf Theorem~\ref{thm:reduction}), which we subsequently use in Sect.~\ref{sec:deficiency} to show existence of projections that strictly reduce deficiency, and
again in Sect.~\ref{sec:examples} to build examples of classes of VC dimension $d$ which cannot be embedded into maximum classes of VC dimension $2d-1$. The characterisation is in terms of iterated reductions.
%We also give an interesting connection to classes which have a property which is similar to being maximal of VC dimension $d$.

\begin{definition}
Consider a $d$-complete collection $C$ embedded in the $n$-cube, a set of $d-1$ directions $S\subset[n]$, and the projection of $C$ onto directions $\overline{S}$. Then the \term{iterated reduction} $C^S$ of $C$ under this projection is the graph $G$ embedded in the $(n-d+1)$-cube with edges the images of $n-d+1$ $d$-cubes of $C$ varying along $S$, nodes the images of the $(d-1)$-faces of directions in $S$, and with a node incident to an edge when (respectively) the corresponding $(d-1)$-face is contained in the corresponding $d$-cube.
%We define the \term{iterated reduction} of a complete collection of $d$-cubes $C$, embedded in an $n$-cube, as the graph $G$ obtained by taking the projection of the $n$-cube to an $(n-d+1)$-cube, where $G$ is the union of $d$-cubes which contain the $d-1$ directions of projection. Equivalently $G$ is the subcollection of $d$-cubes which all share the same set $S$ of $d-1$ coordinate directions, formed into a graph by taking each cube as an edge connecting the two faces of dimension $d-1$ which have the set $S$ of coordinates.
\end{definition}

\begin{figure}[tb]
\sidecaption[t]
\includegraphics[height=1.6in]{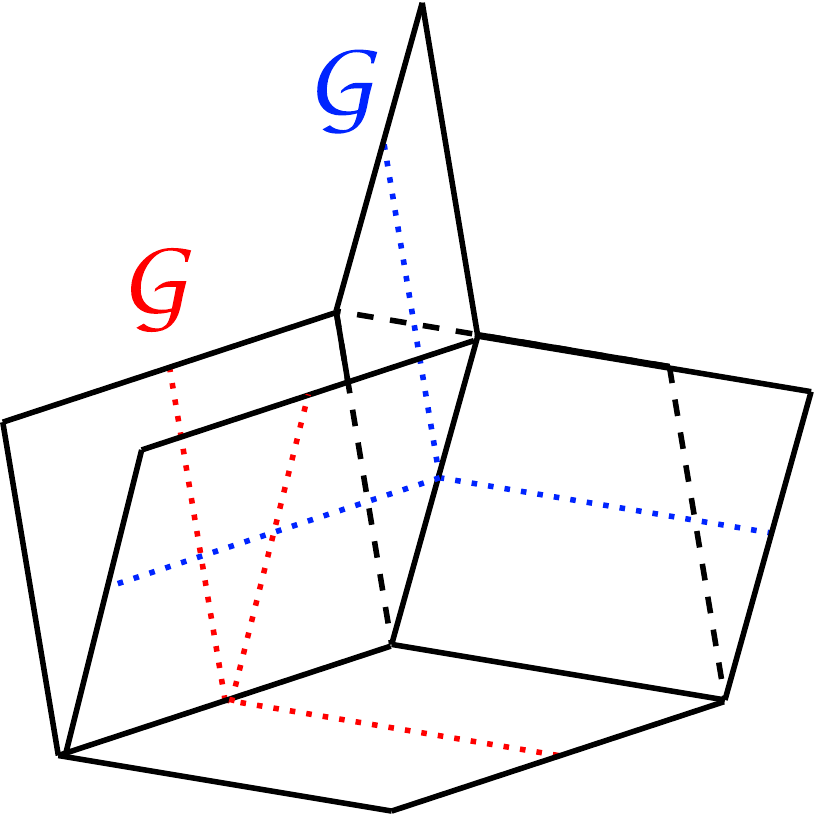}
\caption{The iterated reduction trees of a concept class.}
\label{fig:fig1}
\end{figure}

Figure~\ref{fig:fig1} illustrates the iterated reductions for a class $C$.

\begin{proposition}\label{prop:forest} % n-d-1 --> d
For every class $C$ which is a complete union of $d$-cubes, every $(d-1)$-iterated reduction is a forest.
\end{proposition}

\begin{proof}
Consider a $(d-1)$-iterated reduction $G$ along $d-1$ colors $S$.
Assume $G$ has a cycle. Project out the $d-1$ coordinates corresponding to the colours in $S$. The cycle in $G$ corresponds to a collection of $d$-cubes which project to edges in the binary $(n-d+1)$-cube. Hence there are two such edges of the same colour which come from different $d$-cubes with the same colours. This is a contradiction, since there is only one $d$-cube per choice of colours in $C$.
\qed\end{proof}

\begin{theorem}\label{thm:reduction}
A complete union $C$ of $d$-cubes in the $n$-cube is a maximum class if and only if all the $(d-1)$-iterated reductions are trees, \ie are connected. 
\end{theorem}

\begin{proof}
Firstly, if $C$ is a maximum class, then any reduction is maximum~\cite{W87,KW07}. Now $G$ can be viewed as the result of taking multiple reductions $d-1$ times so is a maximum class of VC-dimension $1$, \ie a tree, proving the necessity of connectedness.

For the converse, we note that a tree has Euler-characteristic one, whereas a forest has Euler characteristic given by the number of trees in the forest (\cf \eg~\cite{Trudeau94}). Therefore if all the iterated reductions are trees, the sum of all their Euler characteristics is the number of iterated reductions, which is clearly ${n \choose d-1}$, since this is the number of ways of choosing a set $S$ of $d-1$ coordinate directions. The Euler characteristic is defined as the number of vertices minus the number of edges of a graph; for the collection of iterated reductions, counting up all the edges gives $d$ times the number of $d$-cubes in a complete collection, which is $d{n \choose d}$, since each $d$-cube is counted $d$ times, one for each pair of $(d-1)$-subcubes with the same collection of $d-1$ coordinates. The total number of vertices in the trees is the number $F$ of $(d-1)$-cubes in $C$. We conclude that $${n \choose d-1}=F-d{n \choose d}$$ if all the iterated reductions are trees. Consequently, this can be rewritten as $$F={n \choose d-1}+d{n \choose d}={n \choose d-1}+(n-d+1){n \choose d-1}$$ which is the expression for the number of $(d-1)$-cubes in a maximum class of VC-dimension $d$ in the $n$-cube by Theorem~\ref{thm:main}. So applying the theorem, we conclude that if all the iterated reductions of a class $C$ are trees, then $C$ is a maximum class. 
\qed\end{proof}

Note that the graph $G$ depends on the choice of the cubical structure of $C$. So if $C$ has different cubical structures, it yields different iterated reductions.
%The previous result Theorem~\ref{thm:reduction} can be interpreted as saying that $C$ is a maximum class if and only if every subgraph $\Gamma_c$ is a tree. In this case, there is a unique way of expressing $C$ as a complete union of $(n-d-1)$-cubes.
The following minor, but novel, result proves that maximum classes have unique iterated reductions.

\begin{lemma}
Any class $C\subseteq\{0,1\}^n$ containing two $d$-cubes of the same set of colors has $\VC{C}\geq d+1$.
\end{lemma}

\begin{proof}
Form a set of $d+1$ colors by taking the $d$ colors of the cubes with any anchor color on which
the two cubes differ. Trivially this set is shattered.
\qed\end{proof}

\begin{corollary}
Let $C$ be a $d$-maximum class. Then $C$ has a unique representation as a $d$-complete collection.
\end{corollary}

\begin{remark}
We note that the set of $(d-1)$-iterated reductions can be integrated into the one structure 
known as the face graph in computational geometry. 
The \term{face graph} $\Gamma$ for a $d$-complete collection $C\subseteq\{0,1\}^n$, is a bipartite graph with vertices for each $d$-cube and each $(d-1)$-cube of $C$. $\Gamma$ has an edge between vertices associated to a $d$-cube and a $(d-1)$-cube, whenever the latter is a face belonging to the former. For any $S\subseteq[n]$ of size $d-1$, define induced subgraph $\Gamma_S$ of $\Gamma$ consisting of all vertices and edges corresponding to cubes whose directions contain $S$.
Then $\Gamma_S$ corresponds to the iterated reduction for directions $S$ subdivided to be made bipartite.
%The vertices of $\Gamma$ corresponding to $(n-d-2)$-cubes can be coloured by the $n-d-2$ coordinate directions of such a cube. The edges of $\Gamma$ can also be coloured by the colours of the incident $(n-d-2)$-cube's vertex. For colors $c\subseteq[n]$ of size $n-d-2$, let $\Gamma_c$ be the induced subgraph of $\Gamma$ consisting of all vertices and edges with color sets containing $c$.
\end{remark}

%The next graph captures similar structure to all the iterated reductions of a class (the relation is made explicit by Remark~\ref{rm:face-iterated}).

%\begin{df}
%The \term{face graph} $\Gamma$ for a class $\overline C$ in the $n$-cube, which is a complete union of $(n-d-1)$-cubes, is a bipartite graph with vertices for each $(n-d-1)$-cube and each $(n-d-2)$-cube of $\overline C$. $\Gamma$ has an edge between vertices associated to an $(n-d-1)$-cube and an $(n-d-2)$-cube, whenever the latter is a face belonging to the former.

%The vertices of $\Gamma$ corresponding to $(n-d-2)$-cubes can be coloured by the $n-d-2$ coordinate directions of such a cube. The edges of $\Gamma$ can also be coloured by the colours of the incident $(n-d-2)$-cube's vertex. For colors $c\subseteq[n]$ of size $n-d-2$, let $\Gamma_c$ be the induced subgraph of $\Gamma$ consisting of all vertices and edges with color sets containing $c$.
%\end{df}

%\begin{rem}\label{rm:face-iterated}
%$\Gamma_c$ corresponds to the iterated reduction subdivided to be made bipartite.
%\end{rem}

%\begin{rem}
%We have expressed Proposition~\ref{prop:forest} in terms of  $\overline C$ since the complementary class $C$ in the binary $n$-cube to $\overline C$ has VC dimension at most $d$ if and only if $\overline C$ contains a complete collection of $(n-d-1)$-cubes. Moreover if $C$ is maximal of VC dimension $d$ then 
%$\overline C$ is a complete collection of $(n-d-1)$-cubes. (See Theorem~\ref{thm:complement}).
%\end{rem}

\section{Deficiency and Embedding VC Classes into Maximum Classes}\label{sec:deficiency}

Our main result in this section is the following;

\begin{theorem}\label{thm:deficiency}
Suppose $C \subseteq \{0,1\}^n$ is a VC-$d$ concept class with deficiency $D$. Then there is an embedding of $C$ into a $(d+D)$-maximum class
$C^\star \subseteq \{0,1\}^n$.
\end{theorem}

The proof of this will follow immediately from two preliminary results, which are of independent interest. 

\begin{proposition}\label{prop:projection}
Suppose $C \subseteq \{0,1\}^n$ is a VC-$d$ concept class and for some $k$, there is a projection $p:\{0,1\}^n \to \{0,1\}^{n-k}$ so that 
$p(C)$ is $d$-maximum. Then there is a $(d+k)$-maximum class $C^\star \subseteq \{0,1\}^n$ so that $C \subseteq C^\star$. 
\end{proposition}

\begin{proof}
The argument is by induction on $k$. Assume first that $k=1$. Since $p(C)$ is maximum, it follows that the complementary class $\overline {p(C)}$ is also maximum by \cite{RBR08}. Consider the inverse image of this complementary class $X=p^{-1}(\overline{p(C)})$. This has the structure of a product ${\overline {p(C)}} \times \{0,1\}$. We observe that there are embeddings of maximum classes of VC dimension $n-d-2$ in $X$. For by the tail-reduction procedure of \cite{KW07}, we can find a maximum VC-$(n-d-3)$ class embedded in the maximum VC-$(n-d-2)$ class $\overline {p(C)}$, as a union of faces of codimension one of the $(n-d-2)$-cubes. By lifting \cite{Rubinsteins09}, we can find many embeddings of maximum VC-$(n-d-2)$ classes in $X$. But then the complement of any such a class is a maximum VC-$(d+1)$ class in the binary $n$-cube containing $C$. This completes the first step of the induction argument. 

Now assume the result is correct for $k-1$. Let $p:\{0,1\}^n \to \{0,1\}^{n-k}$ be a projection and $C$ a VC-$d$ concept class in the binary $n$-cube, so that $p(C)$ is maximum of VC dimension $d$. We factorise $p$ into the composition of projections $p=p' \circ p''$ where $p':\{0,1\}^n \to \{0,1\}^{n-1}$ and $p'':\{0,1\}^{n-1} \to \{0,1\}^{n-k}$. Apply the induction step to the projection $p''$ and the class $p'(C)$. Since $p(C)$ has VC dimension $d$ clearly the same is true for $p'(C)$. We conclude that $p'(C)$ is contained in a maximum class $C^\star$ of VC dimension $d+k-1$.

To complete the proof, we follow the same approach as for the case $k=1$ applied to the image of the complementary maximum class $\overline{C^\star}$ in the binary $(n-1)$-cube. Namely by lifting, we can find maximum classes in $p'^{-1} (\overline{C^\star})$ of VC dimension $n-d-k-1$. The complement of such a class will then be a maximum class in the binary $n$-cube containing $C$ of VC-dimension $d+k$ as required. 
\qed\end{proof}

\begin{proposition}\label{prop:deficit}
Suppose $C \subseteq \{0,1\}^n$ is a VC-$d$ concept class which is not maximum. Then there is a projection $p:\{0,1\}^n \to \{0,1\}^{n-1}$ 
so that $p(C)$ has VC dimension $d$ and deficiency strictly less than the deficiency of $C$. 
\end{proposition}

\begin{proof}
Firstly, since $C$ has VC dimension $d$, there is a $d$-set $S\subseteq[n]$ shattered by $C$. Therefore for any $x\notin S$, the corresponding projection $p_{\overline x}$ from the binary $n$-cube to the binary $(n-1)$-cube maps $C$ onto a VC-$d$ class. The idea is to prove that for one such direction $x$, the deficiency of $p_{\overline x}(C)$ is strictly less than that for $C$.

As in \cite{KW07} we consider the tail/reduction of the projection $p_{\overline x}$ applied to $C$. We consider the image $p_{\overline x}(C)$ and the reduction $C^x$---the subset of the binary $(n-1)$-cube, such that $C^x \times \{0,1\}$ is all pairs of vertices $v_0,v_1 \in C$ with the property that $p_{\overline x}(v_0)=p_{\overline x}(v_1)$. We claim that either the deficiency of $p_{\overline x}(C)$ is strictly less than the deficiency of $C$ or the reduction $C^x$ is a maximum class of VC dimension $d-1$.

To prove the claim, note that the cardinalities of $C,p_{\overline x}(C)$ are related by $|C|=|p_{\overline x}(C)| + |C^x|$. On the other hand, the deficiencies $D,D^\prime$ of $C,p_{\overline x}(C)$ respectively satisfy $D= \sum_{i=0}^{d} {n \choose i} -|C|, D^\prime= \sum_{i=0}^{d} {n-1 \choose i} -|p_{\overline x}(C)|$ respectively. Hence we see that $D-D^\prime = \sum_{i=0}^{d-1} {n-1 \choose i} +|C|-|p_{\overline x}(C)|=\sum_{i=0}^{d-1} {n-1 \choose i} -|C^x|$. But the binomial sum is precisely the cardinality of a maximum VC-$(d-1)$ class in the binary $(n-1)$-cube and hence the difference is positive unless $C^x$ is maximum, by Sauer's lemma, since clearly the VC dimension of $C^x$ is at most $d-1$. This establishes the claim. 

We can now conclude that either the proposition follows, or for each $n-d$ directions $x\notin S$, the corresponding projection $p_{\overline x}$ has reduction $C^x$ for $C$ which is maximum of VC dimension $d-1$. In the latter case, consider an iterated reduction $C^R$ as in Theorem~\ref{thm:reduction}, where $R \cap \overline{S} \ne \emptyset$. It is easy to see that $C^R$ is isomorphic as a graph to an iterated reduction coming from a reduction class $C^x$, so long as $x$ is in $R \cap \overline{S}$. For then we can take the iterated reduction of $C^x$ corresponding to the set of directions $R \setminus \{x\}$ and it follows immediately that the two graphs are isomorphic. But then since $C^x$ is maximum, the corresponding iterated reduction is a tree. This shows that all iterated reductions $C^R$ are trees, so long as $R \cap \overline{S} \ne \emptyset$. 

To complete the proof, we need to deal with the iterated reductions $C^R$, where $R \subseteq [n] \cap S$. This is precisely the initial set of $d$ directions for which $C$ shatters. But since all the reductions $C^x$ are assumed maximum, for $x \notin S$ we see that $C$ shatters all sets of $d$ directions so that $x$ is one of the directions. To see this, note that $C^x$ maximum means that it is a complete union of $(d-1)$ cubes and multiplying by $\{0,1\}$ gives a set of $d$-cubes covering all sets of $d$ directions containing $x$. It is now easy to find new sets $S'$ of $d$ directions shattered by $C$ which do not contain any chosen set $R$ of $d-1$ directions. So the previous argument applies to show that either there is a direction $x$ so that the projection $p_{\overline x}$ reduces the deficiency of $C$ or all possible iterated reductions $C^R$ are trees. In the latter case, $C$ is a maximum class by Theorem~\ref{thm:reduction} and the proof is complete. 
\qed\end{proof}

\begin{theopargself}
\begin{proof}[of Theorem~\ref{thm:deficiency}]
Assume that $C$ is a VC-$d$ class in the binary $n$-cube with deficiency $D$. By repeated applications of Proposition~\ref{prop:deficit}, we can reduce the deficiency of $C$ to zero and hence get a maximum class as image, after at most $D$ projections along single directions. But then by Proposition~\ref{prop:projection}, this implies that there is an embedding of $C$ into a maximum class of VC dimension $d+D$. 
\qed\end{proof}
\end{theopargself}

\section{An Application to Inembeddability}\label{sec:examples}

In this section, we give examples of concept classes $C$ of VC-dimension $d$ which cannot be embedded in any maximum class of VC-dimension $2d-1$. Moreover we exhibit maximum classes of VC-dimension $2d$ which contain each of our classes $C$. This negative result improves previous known examples~\cite{Rubinsteins09} where it was shown that there is no constant $c$ such that any class of VC-dimension $d$ can be embedded in a maximum class of VC-dimension $d+c$. 

\begin{theorem}\label{thm:construction}
There are classes $C$ of VC-dimension $d$ in the binary $n$-cube for each pair $d,n$ satisfying $d$ is even and $n >2d+2$ with the following properties:
\begin{itemize}
\item There is no maximum class $C^\prime$ of VC-dimension at most $2d-1$ in the binary $n$-cube containing $C$.
\item There is a maximum class $C^\prime$ of VC-dimension $2d$ containing $C$ and $C^\prime$ can be taken as a bounded below maximum class, for a suitable choice of origin of the binary $n$-cube. 
\end{itemize}
\end{theorem}
 
\begin{proof}
The proof proceeds by a number of steps. 
 
\myparagraph{Construction of $\mathbf C$} Partition the $n$ coordinates of a binary $n$-cube into sets $A,B$ of size $k,k$ or $k+1,k$, where $n=2k$ or $n=2k+1$ respectively. (In fact, roughly equal size will also work for the construction). We first describe the complement $ \overline C$ to $C$. $ \overline C$ is a complete union of $(n-d-1)$-cubes, the anchors of which are $(d+1)$-strings with the property that each string is either all zeros or all ones. The former is chosen if the majority of the anchor coordinates are in $A$ and the latter if the majority are in $B$. (Having $d$ even means that the anchors are of odd length, so we do not need tie-breaking).

\myparagraph{Computing VC Dimension}  It is immediate that the VC dimension of $C$ is at most $d$. We claim that the VC dimension cannot be less than $d$. If the VC dimension of $C$ was at most $d-1$, there would be a complete collection of $(n-d)$-cubes in the complementary class $\overline C$. We show that this leads to a contradiction. Suppose that $c$ is an $(n-d)$-cube embedded in $ \overline C$. The anchor for  $c$ is of length $d$. Assume $c$ is chosen so that there are exactly $\frac{d}{2}$ elements of the anchor in $A$ and $\frac{d}{2}$ in $B$. Consider an element $v \in c$ which has all the coordinates which are in $A$ but not in the anchor of $c$, having value one and all the coordinates which are in $B$ and not in the anchor, having value zero. As $v \in c \subset C$, it follows that $v$ is in one of the cubes $c_0$ of $C$. $c_0$ must have an anchor either consisting of $d+1$ zeros and the majority of the anchor coordinates must be in $A$ or the anchor has $d+1$ ones, with the majority of the anchor coordinates in $B$. But in both cases, there would be at least $\frac{d+2}{2}$ coordinates of $v$ which are in $A$ or $B$ and are all zeros or ones respectively. This gives a contradiction and we conclude that $c_0$ is not contained in $\overline C$ and hence the VC dimension of $C$ is $d$. 

\begin{figure}[tb]
\sidecaption[t]
\includegraphics[height=1.8in]{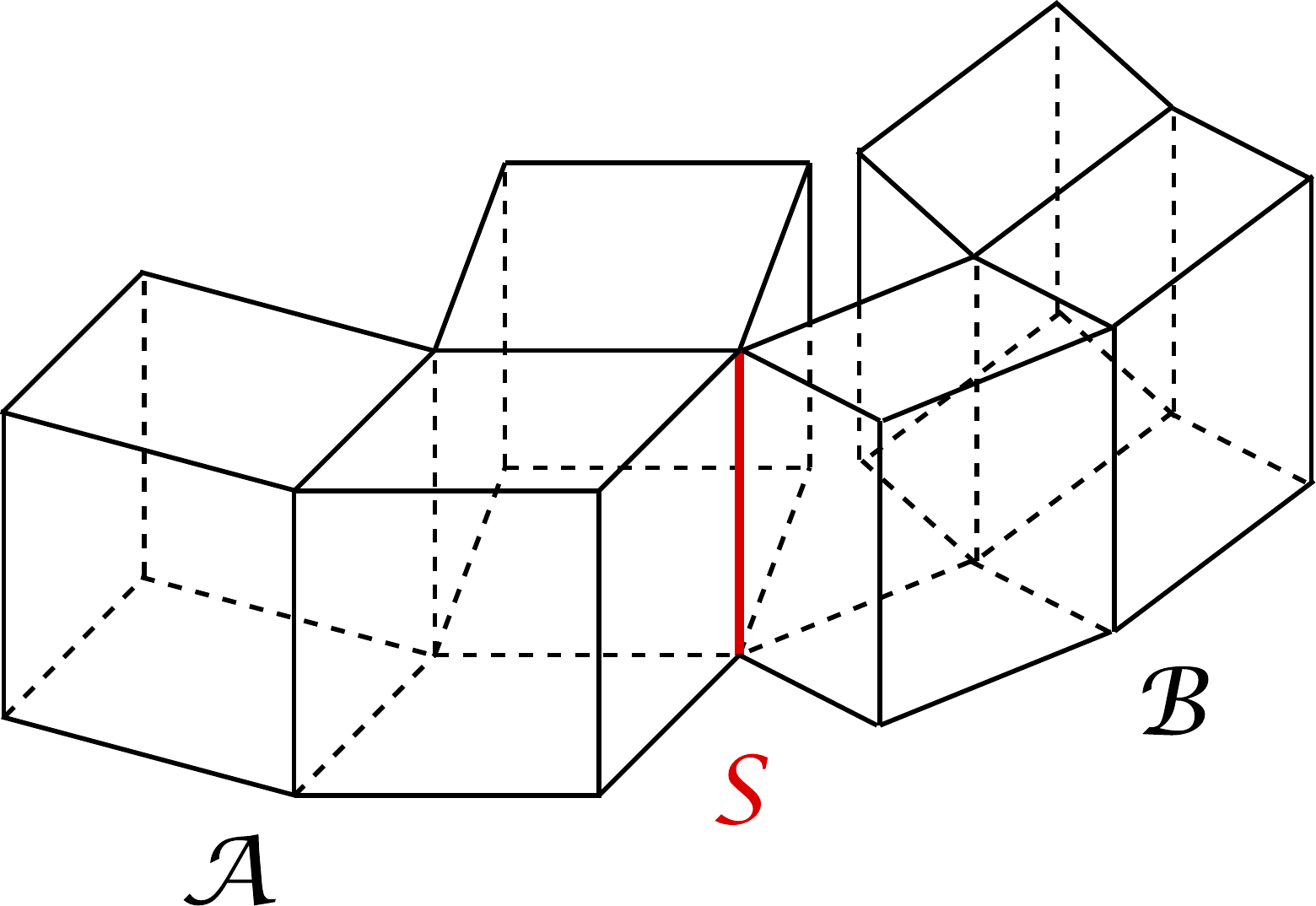}
\caption{Gluing two cubical subcomplexes $\mathcal{A}$ and $\mathcal{B}$ along a single edge $\mathcal{S}$.}
\label{fig:fig2}
\end{figure}

\myparagraph{Decomposing the Complementary Class} Divide $\overline C$ into two sets of cubes, $\mathcal A$ with anchors all zero and $\mathcal B$ with anchors all one.  We abuse notation by using the same symbol for a collection of cubes and also the elements in the unions of these cubes. Note that a pair of cubes, one from each of these two collections, either will be disjoint or will intersect in a cube of dimension $n-2d-2$, depending on whether the anchors have any coordinates in common or not. In particular, $S= \mathcal A \cap \mathcal B$ is a union of $(n-2d-2)$-cubes with anchors consisting of $d+1$ zeros and $d+1$ ones. No two of these cubes have anchors with exactly the same sets of coordinates. So $S$ is a subcollection of a complete collection of $(n-2d-2)$-cubes (\cf Fig.~\ref{fig:fig2}).

We claim there are no $(n-2d-1)$-cubes in $S= \mathcal A \cap \mathcal B$. Recall that any vertex in $S$ belongs to an $(n-2d-2)$-cube with anchor consisting of $d+1$ zeros and $d+1$ ones. But any $(n-2d-1)$-cubes must contain vertices which are not of this form, \eg which have at most $d$ zeros or ones. So this proves that $S$ has no $(n-2d-1)$-cubes. 

\myparagraph{Inembeddability into $\mathbf{(2d-1)}$-Maximum Classes}  We claim that no maximum class of VC dimension at least $n-2d$ can be contained in $\overline C$. Taking complements, this shows that the original class $C$ cannot be contained in a maximum class of VC dimension $\le 2d-1$. By~\cite{RBR08}, a maximum class $M$ of VC dimension at least $n-2d$ inside $\overline C$ is a complete union of cubes. We can assume without loss of generality that $M$ has VC-dimension $n-2d$ since it is well-known that any maximum class contains maximum classes of all smaller VC dimensions. The key step is to show that any $(n-2d)$-dimensional cube of $M$ is contained in either $\mathcal A$ or in $\mathcal B$. Once this is shown, it is easy to deduce a contradiction to the assumption that $M$ is maximum. For if we consider any iterated reduction of $M$ as in the previous section, not all the cubes can lie in $\mathcal A$ say. Hence some are in $\mathcal A$ and some in $\mathcal B$. But these cubes can only meet in $S=\mathcal A \cap \mathcal B$ which is a union of $(n-2d-2)$-cubes. Moreover we have previously shown there are no $(n-2d-1)$-cubes in $S$. Consequently, the assumption that these cubes have faces of dimension $n-2d-1$ in a tree structure for the iterated reduction is contradicted.

Consider an $(n-2d)$-cube $c^\prime$ of $M$.  Now the anchor has $2d$ digits. Clearly the anchor can have at least $d+1$ zeros or at least $d+1$ ones but not both. So without loss of generality, assume the anchor of $c^\prime$ has at least $d+1$ zeros. If the majority of the coordinates corresponding to these zeros are in $A$, then we see that $c^\prime \subset \mathcal A$ as required. Therefore it suffices to suppose that this is not the case, \ie the majority of the coordinates corresponding to the zeros in the anchor of $c^\prime$ are in $B$. But then we get a contradiction, because $c^\prime$ has vertices where all the coordinate entries outside the anchor which are in $A$ are all one and all those in $B$ are zero. For such a vertex clearly does not belong to $C$. We conclude that $c^\prime$ must be in $\mathcal A$ as claimed and the construction is complete. 

\myparagraph{Embedding into $\mathbf{2d}$-Maximum Classes} To show there is a maximum class $M^\star$ of VC-dimension $n-2d-1$ in $\overline C$, define the complete collection of $(n-2d-1)$-cubes of $M^\star$  to have anchors with entries zero for coordinates in $A$ and one for coordinates in $B$. It is easy to see that all these cubes are indeed in $C$, since the anchors are of length $2d+1$, so there must be either at least $d+1$ coordinates in $A$ or $d+1$ coordinates in $B$. Hence $M^\star \subset \overline C$. To see that $M^\star$ is maximum, flip all the coordinates in $B$ interchanging zero and one. Then it follows immediately that $M^\star$ is actually a closed-below maximum class. 
\qed\end{proof}

\section{Embedding of VC-$d$ Classes into VC-$(d+k)$ Maximum Classes}\label{sec:max}

In this section we develop an algorithm that, given a VC-$d$ class $C$ and desired positive integer $k$, builds a $(d+k)$-maximum class containing $C$ if one exists. 
We start by enlarging $C$ such that $\overline C$ is a complete union of $(n-d-1)$-cubes. Our aim is to find a complete union $\overline C^\star$ of $(n-d-k-1)$-cubes inside $\overline C$. The complement $C^\star$ is the required VC-$(d+k)$ maximum class containing $C$.

\begin{algorithm}[t!]
\caption{Compute all maximum embeddings}
\label{alg:embed}
\begin{algorithmic}[1]
\REQUIRE $\overline{C}$ a $(n-d-1)$-complete collection; integer $k > 0$
\STATE Initialise $\mathcal M$ as a queue comprising the closed-below $(n-d-k-1)$-complete collection
\FOR{$i=1$ to $n$}
  \STATE Compute $P$ the projection of $\overline{C}$ onto coordinates $\{1,\ldots,i\}$
  \STATE Initialise $\mathcal{M}'$ as an empty queue of complete collections
  \FOR{$M$ popped from $\mathcal{M}$}
    \STATE Compute $R$ the reduction of $M$ along $i$
    \STATE Compute $J$ the projection of $M$ along $i$
    \STATE Compute $\mathcal{K}$ the connected components of $J$ split by $R$
    \STATE Initialise $\mathcal{B}$ a set of iteratively built maximum classes to $\emptyset$
    \FOR{connected component $K$ in $\mathcal{K}$}
      \STATE Initialise $S$ set of connected components as $\emptyset$ 
      \FOR{$j=1$ to $2$}
        \IF{the projection of $K\times \{j\}$ on $\{1,\ldots,i\}$ is contained in $P$}
          \STATE Update $S = S \cup \{K\times \{j\}\}$
        \ENDIF
      \ENDFOR
      \STATE \textbf{if} $S=\emptyset$ \textbf{then} Exit loop jump to 5
      \STATE Update $\mathcal{B} = \mathcal{B} \times S$
    \ENDFOR
    \STATE Push the maximum classes built up in $\mathcal{B}$ onto $\mathcal{M}'$
  \ENDFOR
  \STATE Swap $\mathcal{M}$ with $\mathcal{M}'$
\ENDFOR
\RETURN $\mathcal{M}$
\end{algorithmic}
\end{algorithm}

Algorithm~\ref{alg:embed} aims to produce all $(d+k)$-maximum classes containing $C$.
The output of the algorithm $\mathcal{M}$ is this set, and is empty if no such classes
exist.
The strategy, working in the complement as usual, proceeds iteratively from the
canonical closed-below $(n-d-k-1)$-maximum class. At each iteration the next dimension in
$[n]$ is considered: components of the $(n-d-k-1)$-maximum classes from the
previous iteration are lifted along the chosen dimension to eventually be contained within
$\overline{C}$. In particular, we consider embedding in \emph{the dimensions processed so
far}---we check whether the lifted connected component projected onto these dimensions
is contained in $\overline{C}$ also projected. If a choice along the current dimension
achieves containment then
the class is retained; if both choices are feasible then the class is cloned with  
siblings making each choice; if neither choice is possible then the maximum class is 
discarded.

Essentially the process is one of lifting to build arbitrary maximum classes as
developed by Rubinstein \& Rubinstein~\cite{Rubinsteins09}---recall that a complete
collection is lifted by arbitrarily setting the `height' of components of
cubes that are connected without crossing the reduction
(\cf Sect.~\ref{sec:background}). The difference is that we iteratively filter out
intermediate maximum classes as soon as it is clear they cannot be embedded in
$\overline{C}$. 

\begin{proposition}
For any VC-$d$ class $C$ in the $n$-cube, and any $k>0$, Algorithm~\ref{alg:embed} returns
the set of all $(d+k)$-maximum classes in the $n$-cube containing $C$.
\end{proposition}

\begin{proof}
The result follows from the maximum property being invariant to
lifting, lifting constructs all maximum classes of given
dimension~\cite{Rubinsteins09}, and that the algorithm filters out
exactly non-embedded classes as subsequent liftings do not alter the containment
property of earlier iterations.
\qed\end{proof}

\section{VC-$2$ Classes}\label{sec:VC-2}

We study VC-$2$ classes embedded in the binary $n$-cube, for $4 \le n \le 6$. We will prove some results on embedding of these VC-$2$ classes into maximum classes and also on the deficiency of maximal VC-$2$ classes.  Our choices of $d, n$ in this section yield the simplest ``complete picture'' for VC classes for which embedding (and compression) is non-trivial, and as such serve as useful tests for the tools developed above. In particular, we calculate the maximin VC dimension of the maximum classes in which maximal classes are embeddable, as summarised in Table~\ref{tab:vc2}.

\begin{table}
\caption{For $n\in\{4,5,6\}$, the smallest $d$'s such that all $2$-maximal classes in the $n$-cube embed in a $d$-maximum class, and some $2$-maximal class(es) does not embed in a $(d-1)$-maximum class.}
\label{tab:vc2}
\centering
\begin{tabular}{|c|c|}
\hline
$\mathbf n$ & \textbf{maximin $\mathbf d$ maximum-embeddable} \\
\hline
\; 4 \; & 3 \\
5 & 4 \\
6 & 4 \\
\hline
\end{tabular}
\end{table}

\textbf{Case $\mathbf{n=4}$.}
We first classify maximal VC-$2$ classes in the binary $4$-cube and prove these have deficiency $1$. As a corollary it follows that these classes project to maximum VC-$2$ classes in the binary $3$-cube. 

The argument is straightforward. The complement $\overline C$ of a maximal VC-$2$ class $C$ is a complete union of $1$-cubes \ie edges in the binary $4$-cube. Note that such a complete union is maximum if and only if it is a tree. In this case, $C$ too is maximum and so we are not interested in this (trivial) case. Consider then $\overline C$ a forest, with four edges. There are two possibilities: one is that there are two components of size $1,3$ and the other is that there are two components, each of size $2$. (We will verify that having three or more components is not possible). Notice that the components of this forest must be distance at least two apart. Since the diameter of the binary $4$-cube is $4$, it is easy to check that there cannot be three or more components and the two components are either a tree with a vertex of degree $3$ and a single edge, or two trees with two edges each. It is then straightforward to verify that up to symmetry of the 4-cube, there are precisely one of each type of forest. Hence there are precisely two maximal VC-$2$ classes in the binary $4$-cube and both have deficiency $1$. The latter holds since the forests both have one more vertex than a tree, corresponding to the complement of a maximum class. This completes the discussion in the $4$-cube. 

\textbf{Case $\mathbf{n=5}$.}
In the binary $5$-cube, there is a large number of possibilities for a maximal VC-$2$ class. However by our argument in the inembeddability section, it follows that there are VC-$2$ classes which do not embed in VC-$3$ maximum classes in the binary $5$-cube. Since a maximum VC-$4$ class is obtained by removing a single vertex from the binary $5$-cube, it follows immediately that every VC-$2$ class embeds in a maximum VC-$4$ class. But this is clearly a trivial result. 

\textbf{Case $\mathbf{n=6}$.}
Finally let's examine the more interesting case of VC-$2$ classes $C$ in the binary $6$-cube. We claim there is a simple argument that these all embed in maximum VC-$4$ classes. The idea is as usual, to study the complementary class $\overline C$. We can assume this is a complete union of $3$-cubes, by enlarging $C$ if necessary, but not increasing its VC dimension. Consider two such $3$-cubes $C_1, C_2$ with anchors at disjoint sets of coordinates $S_1, S_2$. Note that $C_1 \cap C_2$ contains the vertex $v$ with coordinate values at $S_1$ (respectively $S_2$) given by the anchor of $C_1$ (respectively $C_2$). Hence there is a tree $\Gamma$ embedded in $C_1 \cup C_2$ consisting of six edges one of each coordinate type, with three in $C_1$ all sharing $v$ and three in $C_2$ all containing $v$. But then the complementary class $\overline {\Gamma}$ is a maximum class of VC dimension $4$ containing $C$. (In fact it is easy to see that if the coordinates of the binary cube are flipped so all the coordinates of $v$ are $1$, then $\overline \Gamma$ is actually closed-below maximum.)

\section{Boolean Functions}\label{sec:functions}

Our aim in this section is to consider special VC-classes corresponding to Boolean functions and study their associated maximum classes. In Ehrenfeucht \etal~\cite{EHKV89}, Procaccia \& Rosenschein~\cite{monotone}, the learnability of examples of such classes are considered by way of computing VC dimensions. We will show that there are interesting connections between natural classes of Boolean functions and maximum classes, hence yielding information about compression schemes for such classes.  We begin with symmetric functions, showing the class can be enlarged to a maximum class of the same VC dimension. We then show that using a suitable basis of monomials, classes of Boolean functions can be formed by sums, which are maximum classes of arbitrary VC dimension.

\subsection{Symmetric Functions}

\begin{definition}
A function $f:\{0,1\}^n \to \{0,1\}$ is \emph{symmetric} if it has the same value when coordinates are permuted.
\end{definition}

We study the class of symmetric functions $\mathcal F \subset \{0,1\}^{\mathcal X}$ where $\mathcal X$ is the binary $n$-cube $\{0,1\}^n$. Each symmetric function $f: \mathcal X \to \{0,1\}$ is associated to the mapping given by $x \mapsto f(x)$ where $x \in \mathcal X$ is a binary $n$-vector. Clearly a symmetric function is completely determined by the number of coordinates with value $1$ which are in vectors mapped to $1$. 

We introduce some notation to assist the discussion. Coordinates in $\mathcal X$ will be the \emph{monomials} $\emptyset, x_1, \dots, x_n, x_1x_2, \dots, x_{n-1}x_n, \dots, x_1x_2, \dots, x_n$. Here the variable $x_i$ indicates a $1$ in the \ith{i}{th} location of a binary $n$-vector. We divide the coordinates into $n+1$ classes $S_0, S_1, \dots S_n$ so that each class consists of all monomials of the same degree (matching the class index). Then a symmetric function $f$ has the same value on all monomials in each class $S_i$. There are therefore $n+1$ \emph{degrees of freedom} of functions in $\mathcal F$. 

We prove the following result due to Ehrenfeucht \etal~\cite{EHKV89} via a novel argument that leverages the class's
natural structure under the above partitioned-monomial basis.

\begin{lemma}
The VC dimension of $\mathcal F$ is $n+1$.
\end{lemma}

\begin{proof}
Using our basis of partitioned monomials, it is easy to see that the VC dimension of $\mathcal F$ is at least $n+1$. For we can choose symmetric functions which evaluate independently on each of our $n+1$ classes $S_i$ of monomials. Hence we see that $\mathcal F$ shatters a set $S$ of $n+1$ coordinates, so long as there is one coordinate from each class $S_i$ in $S$. On the other hand, it is also easy to see that there is no shattering of an $(n+2)$-set. For if we choose any collection of $n+2$ coordinates, then two of them have to be in the same class $S_i$. Hence every element of $\mathcal F$ does not distinguish these two coordinates, so shattering does not occur. This establishes that the VC dimension of $\mathcal F$ is exactly $n+1$.
\qed\end{proof}

Next, consider the collection of $(2^n-n-2)$-cubes in the complement $\overline {\mathcal F}$ of $\mathcal F$. We 
trivially have the following.

\begin{lemma}
The complement $\overline{\mathcal F}$ contains a complete collection of $(2^n-n-2)$-cubes with anchors having $n+2$  coordinates with at least two falling in the same $S_i$ class with differing values.
\end{lemma}

%Notice these have anchors of length $n+2$. Hence for each such cube, there are always at least two anchor coordinates in the same class $S_i$ for some $i$. To get an anchor of a cube in $\overline {\mathcal F}$, it is clear that two anchor coordinates in the same class $S_i$ must differ. This gives a necessary and sufficient condition for the anchors of such cubes. 

Finally we establish the following novel result on the maximum-embedding of the class of symmetric Boolean functions.

\begin{proposition}
There exists a maximum class of VC dimension $n+1$ containing $\mathcal F$.
\end{proposition}

\begin{proof}
Choose an ordering of the monomial coordinates of $\mathcal X$ consistent with their degrees. So if a monomial $m$ has larger degree  than a monomial $m^\prime$ then $m > m^\prime$ in the ordering. 

The complement $\overline {\mathcal M}$ of $\mathcal M$ is a complete collection of $(2^n - n-2)$-cubes with anchors of length $n+2$. We describe the set of anchors of these cubes. 

Each anchor has $n+1$ coordinates set equal to $0$ and a single coordinate equal to $1$. The special coordinate is defined as follows.

For every anchor, there must be at least two anchor coordinates in the same class $S_i$. Choose the first coordinate $m$ in the ordering in $S_i$ for any $i$, where there is a second anchor coordinate $m^\prime$ in $S_i$, and put the value of $m$ equal to $1$. This gives anchors of a complete collection of $(2^n - n-2)$-cubes. 

To show that $\overline {\mathcal M}$ is a maximum class, we study its iterated reductions. This involves a number of cases. 

\myparagraph{Case 1} 
Consider an iterated reduction of $\overline {\mathcal M}$, along a set $S$ of $2^n - n-3$ coordinates. Let $\overline S$ denote the complementary set of $n+3$ coordinates. In the first case, there  are two coordinates $m, m^\prime$ in ${\overline S} \cap S_i$, where $S_i$ the first class with more than one coordinate of $\overline S$ in the ordering . Then there must be at least two coordinates in ${\overline S} \cap S_j$ for $i \ne j$ and $S_j$ is the next class in the ordering containing more than one coordinate of $\overline S$. Each anchor for a cube in the iterated reduction along $S$ has $n+2$ coordinates, forming a set leaving out precisely one element of $\overline S$. There are two possibilities. The first is that the missing coordinate is not in $S_i$. It is easy to see that the set $\mathcal C$ of all such cubes overlap in pairs in codimension one faces. So it remains to consider what happens for the remaining cubes. Clearly there are two such cubes, say $C_1, C_2$. Both $C_1,C_2$ have a $0$ in the single remaining coordinate in $S_i$. Assume that the coordinate of $C_1$ in $S_i$ occurs before the coordinate of $C_2$ in $S_i$ in the ordering. $C_1,C_2$ also have a $1$ in $S_j$, since this now becomes the first class in the ordering where there are multiple anchor coordinates for the cubes. It is not difficult to see that $C_2$ has a codimension one face in common with a cube of $\mathcal C$. Moreover $C_1$ has a codimension one face in common with $C_2$. Hence it follows that the iterated reduction is a tree. 

\myparagraph{Case 2} 
Suppose that there are at least three coordinates of $\overline S$ in the first class $S_i$ in the ordering with more than one coordinate of $\overline S$ in $S_i$. It is not difficult to again enumerate cases and see that the cubes with anchors obtained from $\overline S$, by leaving out one of the coordinates of ${\overline S} \cap S_i$, have codimension one faces in common. Finally if we leave out one of the remaining coordinates of $\overline S$ , it is obvious that these cubes meet in pairs of codimension-one faces. Moreover it is easy to find a cube from the first family and one from the second which have a codimension-one face in common. So this completes the argument that $\overline {\mathcal M}$ is maximum and hence $\mathcal F$ embeds in $\mathcal M$, which is maximum of VC dimension $n+1$. 
\qed\end{proof}

\subsection{A Method for Generating Maximum Boolean Function Classes}

We next provide a method to generate interesting collections of Boolean functions which form maximum classes. We start with degree $n$ monomials in the binary $n$-cube. These are expressions of the form $a_1 \wedge a_2 \wedge \dots a_n$ where each $a_i$ is either $x_i$ or $\neg x_i$. We wish to find a collection $\mathcal B$ of Boolean functions, which is a maximum class of VC dimension $k$ in the binary $2^n$-cube. We begin with a \emph{generating set} for $\mathcal B$.
This is an ordered set $\mathcal G$ given by $2^n$ sums of distinct $n$-monomials, denoted $s_1, s_2, \dots s_{2^n}$:
\begin{itemize}
\item $s_1$ is any single monomial; and
\item Each subsequent $s_i$ has a unique representation as the sum of a single monomial and $s_j$ for some $j < i$.
\end{itemize}

The following is easy to verify.

\begin{lemma}
The set $\mathcal G \cup \{\emptyset\}$ is a maximum class of VC-dimension 1 in the $2^n$-cube, where $\mathcal G$ is a generating set $\{s_1, s_2, \ldots, s_{2^n}\}$ and $\emptyset$ is the zero Boolean function. 
\end{lemma}

We now may build $\mathcal B$ by taking all sums of zero up to $k$ distinct elements from the set $\{s_1, s_2, \dots s_{2^n}\}$. It follows that $\mathcal B$ is maximum.

\begin{proposition}
$\mathcal B$ is maximum of VC dimension $k$.
\end{proposition}

\begin{proof}
First, it is clear that the cardinality of $\mathcal B$ is $\sum_{i=0}^{k} {2^n \choose i}$. For if two sums are equal, then by Boolean addition, we obtain that a non-trivial sum is the zero function. But this is clearly impossible by our choice of the generating set as linearly-independent functions over $\mathbb Z_2$. So if we can prove that $\mathcal B$ has VC dimension at most $k$, by Sauer's Lemma it follows that $\mathcal B$ is maximum. 

Consider the projection of $\mathcal B$ to a $(k+1)$-cube. Notice that the projection of the generating set for $\mathcal B$ is a maximum VC-$1$ class $\mathcal C$ in this cube. Hence the projection of $\mathcal B$ consists of all sums of up to $k$ elements of $\mathcal C$. But a maximum VC-$1$ class $\mathcal C$ containing the origin $\tilde 0$ is easily seen to give a basis $\mathcal C \setminus \{\tilde 0\}$ for a binary cube considered as a $\mathbb Z_2$-vector space. Hence in the binary $(k+1)$-cube, the collection of all sums of up to $k$ elements from $\mathcal C \setminus \{\tilde 0\}$ clearly does not contain the element $c_1 +c_2 + \dots c_{k+1}$. Hence this shows the projection of $\mathcal B$ to any $(k+1)$-cube is not onto and so $\mathcal B$ is maximum as claimed.
\qed\end{proof}

\section{Conclusion}\label{sec:conc}

This chapter makes two main contributions. The first is a simple scheme to embed any VC-$d$ class into a maximum class of VC dimension $(d+D)$ where
$D$ is the deficiency. Therefore, for a collection of VC-$d$ classes in binary $n$-cubes, with $n$ increasing, so long as there is a bound on the deficiency of the classes independent of $n$, then the resulting compression scheme from embedding into VC-$(d+D)$ maximum classes satisfies the Sample Compression conjecture of Littlestone \& Warmuth. This focusses attention on maximal VC-$d$ classes, where the deficiency grows with the dimension $n$ of the binary cube. 

Our second main contribution is a negative embeddability result, placing a fundamental limit on the leading approach to
resolving the Sample Compression conjecture---an approach that requires the
embedding of general VC-$d$ classes into $O(d)$-maximum classes. We exhibit VC-$d$
classes that can be embedded into VC-$2d$ maximum classes but not into any VC-$(2d-1)$ maximum class.

We developed our negative result as an application of a generalised Sauer's Lemma,
proved first by Kuzmin \& Warmuth~\cite{KW07}, from bounding the number of points in
a concept class to bounding all hypercubes from edges to faces. We also offer a novel
proof of this result building on recent geometric characterisations as cubical
complexes~\cite{Rubinsteins09}. 

We believe that our negative examples may be close to worst possible. We offer a new iterated-reduction
characterisation that provides a practical approach to measuring whether a union of cubes
is maximum; and we develop an algorithm for building all maximum-embeddings of a given
VC-class. It is our hope that these three new tools may help in embedding all VC-$d$ classes
into maximum classes of dimension $O(d)$ but at least $2d$. As a first step we demonstrate their application to VC-$2$ classes in the 4,5,6-cubes, and also consider maximum-embeddings of classes of Boolean functions.

%the Springer BibTeX styles "spbasic.bst", "spmpsci.bst", "spphys.bst"
\bibliographystyle{spmpsci}
\bibliography{Rubinsteins-Bartlett-2014}

\end{document}